\documentclass[10pt,twocolumn]{article}

\usepackage[margin=1in]{geometry}

\usepackage{microtype}
\usepackage{graphicx}
\usepackage{subcaption}
\usepackage{booktabs}
\usepackage[colorlinks=true,citecolor=blue,linkcolor=blue,urlcolor=blue]{hyperref}
\usepackage{natbib}

\usepackage{amsmath}
\usepackage{amssymb}
\usepackage{amsfonts}
\usepackage{amsthm}

\usepackage{enumitem}

\usepackage{pifont}
\usepackage{times}

\newtheorem{theorem}{Theorem}

\newtheorem{property}{Property}

\newcommand{\cmark}{\ding{51}}
\newcommand{\xmark}{\ding{55}}

\begin{document}

\title{\textbf{Cost-Efficient Multimodal LLM Inference via Cross-Tier GPU Heterogeneity}}
\author{Donglin Yu\\donglin5@illinois.edu}
\date{}
\twocolumn[\maketitle
\begin{abstract}
\noindent
Multimodal large language model (MLLM) inference splits into two phases
with opposing hardware demands: vision encoding is compute-bound, while
language generation is memory-bandwidth-bound. We show that under standard
transformer KV caching, the modality boundary (between vision encoder and
language model) minimizes cross-device transfer among all
partition points that preserve standard stage-based execution. Partitioning here reduces transfer complexity from
$O(L \cdot s_{\text{ctx}})$ bytes (GB-scale KV caches under stage-level
disaggregation) to $O(N_v \cdot d)$ bytes (MB-scale embeddings), an $O(L)$
reduction where $L$ is the transformer depth. The result holds across
attention mechanisms (MHA/GQA), dynamic vision resolutions, and model
scales, and the advantage grows as models deepen. A direct implication is
that existing stage-level disaggregation systems are constrained to
high-bandwidth interconnects (e.g., NVLink), whereas modality-level
disaggregation enables cross-tier heterogeneous serving over commodity
PCIe. A closed-form cost model shows that
heterogeneous deployment is cost-optimal under phase-separable workloads
(predicts 31.4\% savings; observed 40.6\%). We build HeteroServe, a
phase-aware runtime with modality-level partitioning and cross-tier
scheduling, and evaluate it on LLaVA-1.5-7B and Qwen2.5-VL against
vLLM~v0.3.0. On identical 4$\times$A100 hardware, engine optimizations
raise throughput by up to 54\%. Under a fixed budget, a heterogeneous
cluster (\$38k) improves Tokens/\$ by 37\% over a homogeneous baseline
(\$64k) without degrading latency.
\end{abstract}
\vspace{1em}
]

\section{Introduction}

Multimodal large language model (MLLM) inference exhibits an architectural
mismatch: vision encoding and language decoding have different hardware
optimality regimes. Vision encoding is
compute-bound, saturating FP16 tensor cores with negligible memory
bandwidth demand. Language decoding is memory-bandwidth-bound, streaming
model weights and KV caches from HBM with minimal arithmetic intensity.
These two phases stress completely different GPU resources, yet all deployed
MLLM serving systems execute both on \emph{homogeneous} datacenter
hardware, paying a well-documented ``HBM tax''~\citep{cauchy2025,hetis2025}:
the compute-bound phase wastes expensive high-bandwidth memory, while the
bandwidth-bound phase underutilizes tensor cores. In practice, providers must
overpay for HBM even when the current workload is encoder-heavy, because
the entire pipeline is locked to a single GPU tier.

Existing disaggregation systems attempt to address inference efficiency by
partitioning at \emph{pipeline stage boundaries}, i.e., separating prefill
from decode (EPD~\citep{epd2025}, Cauchy~\citep{cauchy2025}). However, stage-level
partition incurs KV cache transfer cost proportional to model depth:
$O(L \cdot s_{\text{ctx}})$ bytes (hundreds of megabytes to gigabytes per
request). This GB-scale transfer limits scalability across heterogeneous GPU tiers,
because it demands NVLink or InfiniBand and
confines all prior disaggregated serving to homogeneous datacenter
environments. Intra-node scheduling (SpaceServe~\citep{spaceserve2025},
UnifiedServe~\citep{unifiedserve2025}) improves utilization through kernel
co-location but does not change the cross-device communication structure and
therefore cannot unlock cross-tier deployment.
The bottleneck is not disaggregation itself but
\emph{where the graph is cut}. No existing system exploits the modality
boundary for heterogeneous serving.

Why has stage-level partitioning gone unquestioned? We argue that it
reflects a \emph{language-only assumption}: in text-only LLMs, the KV cache
is the dominant intermediate state and stage boundaries are the only
meaningful partition points. Multimodal inference violates this assumption: the vision encoder produces
a compact embedding of size
$O(N_v \cdot d)$ (MB-scale), \emph{independent of transformer depth $L$}.
Yet no prior system has exploited this separation for cross-device serving.

We formalize this observation as Theorem~\ref{thm:transfer}: under standard
transformer KV caching (i.e., without activation recomputation or KV
offloading), the modality boundary (between vision encoder and language
decoder) minimizes cross-device transfer complexity among
partition points that preserve standard stage-based execution semantics. The transfer ratio between stage-level and
modality-level disaggregation scales as $\Theta(L)$, corresponding to
$12\times$--$196\times$ across current MLLMs. The result holds independent
of hidden dimension, attention mechanism (MHA or GQA), vision token count,
and model scale, and the advantage grows with model depth.

In practice, this reduces cross-device transfer from GB-scale to MB-scale,
making cross-tier heterogeneous serving over commodity PCIe feasible.
Consumer GPUs (RTX~4090, \$3k, 330~TFLOPS) can handle vision encoding while
datacenter GPUs (A100, \$16k, 2~TB/s HBM) handle language decoding,
a $3\times$ FLOPs-per-dollar advantage that is inaccessible under
stage-level disaggregation. A closed-form cost model further shows that
heterogeneous deployment is cost-optimal under phase-separable workloads
(predicts 31.4\% savings; we observe 40.6\%).

We build HeteroServe, a phase-aware runtime with embedding-only transfer,
cross-type work stealing for idle-capacity recovery, and standard engine
optimizations.

We make the following contributions:

\begin{itemize}
\item \textbf{Transfer Optimality Analysis
  (\S\ref{sec:characterization}--\S\ref{sec:generalization}).}
  We characterize the phase separation in MLLM inference
  (\S\ref{sec:characterization}) and prove
  (Theorem~\ref{thm:transfer}) that, under standard KV caching, the
  modality boundary minimizes cross-device transfer under standard
  stage-based execution, reducing communication by $O(L)$ compared to
  stage-level disaggregation, with empirical ratios of
  $12\times$--$196\times$ across five architectures. We further derive a
  closed-form cost model (\S\ref{sec:cost}) showing that heterogeneous
  deployment is cost-optimal under phase-separable workloads (predicts
  31.4\% savings; we observe 40.6\%).

\item \textbf{System Validation: HeteroServe (\S\ref{sec:system}).}
  We build HeteroServe, a phase-aware runtime that validates the analysis
  on real hardware, with embedding-only transfer over PCIe, cross-type
  work stealing for utilization recovery, and CUDA-Graph-accelerated
  decoding.

\item \textbf{Empirical Validation (\S\ref{sec:experiments}).}
  On LLaVA-1.5-7B (MHA) and Qwen2.5-VL (GQA, tensor-parallel) with
  vLLM~v0.3.0 baseline: on identical 4$\times$A100 hardware, engine
  optimizations raise throughput by up to 54\%
  (\S\ref{sec:disentangle}); under a fixed budget, modality-level
  heterogeneous deployment (\$38k) improves Tokens/\$ by 37\% over a
  homogeneous baseline (\$64k) without degrading latency.
\end{itemize}

\noindent Engine optimizations (CUDA Graph, packed prefill, lazy KV
allocation) serve as implementation hygiene to avoid confounding the
architectural evaluation; they are not claimed as contributions.

\section{Multimodal Inference Phase Characterization}
\label{sec:characterization}

We begin by characterizing the computational phases of MLLM inference and
their hardware implications.

The end-to-end latency of a single multimodal request decomposes as
$T_{\text{e2e}} = T_{\text{vision}} + T_{\text{xfer}} + T_{\text{prefill}} + T_{\text{decode}}$.
Each phase has a distinct resource bottleneck.
\textbf{Vision encoding} ($T_{\text{vision}}$) is \emph{compute-bound}:
on an RTX~4090, batched encoding achieves $>$85\% FP16 TFLOPS utilization
(${\sim}280$ of 330~TFLOPS) while consuming $<$5\% of memory bandwidth.
The compute cost is fixed per image and independent of language model
depth~$L$.
\textbf{Language decoding} ($T_{\text{decode}}$) is
\emph{memory-bandwidth-bound}: autoregressive generation streams weights and
KV caches from HBM with minimal arithmetic reuse, achieving $>$80\% HBM
utilization on A100 (${\sim}1.6$ of 2~TB/s) while using $<$10\% of FP16
compute.
\textbf{Language prefill} ($T_{\text{prefill}}$) is compute-bound but
operates on language model weights.

These bottlenecks map to \emph{orthogonal hardware optimality regimes}.
Vision encoding maximizes performance on high-TFLOPS, low-cost GPUs: an
RTX~4090 (\$3k, 330~TFLOPS) matches an A100 (\$16k, 312~TFLOPS) at
$3.3\times$ better FLOPs/\$. Language decoding requires high-bandwidth GPUs:
A100 provides 2~TB/s HBM with 80~GB; RTX~4090 has only 1~TB/s with 24~GB.
Homogeneous deployments cannot resolve this mismatch, motivating a formal
analysis of \emph{where} to partition the pipeline.

\section{Related Work}

We organize prior work along two axes that jointly define the design space
of disaggregated multimodal serving: (1)~\emph{partition granularity}, i.e., where
the inference graph is cut (stage boundary, modality boundary, or kernel
co-location); and (2)~\emph{deployment regime}, i.e., whether the system supports
cross-device, cross-tier (PCIe), or only intra-node (NVLink/IB) execution.
As Table~\ref{tab:comparison} summarizes, no existing system simultaneously
achieves modality-level partitioning, cross-tier heterogeneous support, and
formal transfer analysis.

\paragraph{LLM serving and memory optimization.}
Systems such as vLLM~\citep{kwon2023efficient,vllm2024} have
advanced memory management and continuous batching for autoregressive
decoding on homogeneous hardware. These optimizations are complementary to
ours; HeteroServe builds on similar batching and KV-cache management
techniques. However, these systems are designed for text-only models and
homogeneous execution pipelines: they do not exploit the resource
heterogeneity intrinsic to multimodal inference, do not analyze partition
point optimality, and do not consider cross-tier GPU deployment. We use
vLLM~v0.3.0 as our primary baseline under identical hardware and
parallelism configurations; the throughput difference we observe
(\S\ref{sec:disentangle}) arises from architectural disaggregation and
engine optimizations, not from baseline misconfiguration.

\paragraph{Stage-level disaggregation.}
A growing line of work disaggregates LLM inference at \emph{pipeline stage
boundaries}. EPD~\citep{epd2025} separates multimodal inference into
encoding, prefill, and decode stages across devices, introducing dynamic
role switching to rebalance capacity. Cauchy~\citep{cauchy2025} optimizes
cost-efficiency by selecting heterogeneous GPU ``combos'' for prefill-decode
disaggregation. However, all stage-level approaches require transferring full KV caches ($O(L \cdot
s_{\text{ctx}})$ bytes, GB-scale) between stages, demanding NVLink or
InfiniBand interconnects. This transfer bottleneck confines these systems to
datacenter-grade links and precludes participation of consumer GPUs connected
via PCIe. Consequently, their conclusions do not extend to cross-tier
heterogeneous deployments. Our work shows that this limitation is not
inherent to disaggregation itself, but rather a consequence of partitioning
at the \emph{wrong boundary}.

\paragraph{Multimodal serving systems.}
ModServe~\citep{modserve2025} separates vision and language into independent
instance pools to mitigate head-of-line blocking, and
CornServe~\citep{cornserve2025} generalizes this via offline deployment
planning across monolithic and disaggregated execution graphs. Both
partition at or near the modality boundary, but assume \emph{homogeneous}
datacenter GPUs for all pools, forgoing the cost advantage of consumer
hardware. Neither identifies modality boundary optimality nor
analyzes transfer complexity scaling, the theoretical foundation that
enables cross-tier deployment. SpaceServe~\citep{spaceserve2025} improves
intra-GPU utilization through spatial multiplexing of complementary
compute-bound and memory-bound kernels, reducing TPOT when encoder and
decoder workloads are co-located. UnifiedServe~\citep{unifiedserve2025}
achieves similar goals via MPS-based resource sharing. In particular,
SpaceServe's optimization target is single-machine SM/bandwidth utilization;
it does not attempt, and structurally cannot, solve the KV cache
cross-device migration problem, and therefore cannot unlock consumer-GPU
participation via PCIe. These systems optimize \emph{intra-node} scheduling
(axis~2: NVLink/IB only) under fixed hardware topologies but do not address
\emph{cross-device} partition optimality or cross-tier deployment,
the orthogonal axis that our work targets.

\paragraph{Heterogeneous GPU serving.}
Hetis~\citep{hetis2025} manages heterogeneous clusters through head-level KV
cache partitioning for text-only LLMs, but does not address the multimodal
pipeline or exploit the compute-bound nature of vision encoding.
Cauchy considers GPU heterogeneity but inherits stage-level KV transfers
($O(L \cdot s_{\text{ctx}})$), requiring high-bandwidth interconnects.
Neither system exploits the observation that multimodal inference
contains a \emph{compute-bound, bandwidth-agnostic} phase (vision encoding)
that can be offloaded to consumer GPUs with only PCIe connectivity. Our work
identifies this asymmetry and provides the first formal cost model
for cross-tier multimodal serving.

\begin{table}[h]
\centering
\caption{Comparison of disaggregated MLLM serving systems. Our work is the
  only system combining modality-level partitioning with cross-tier support
  and formal cost modeling.}
\label{tab:comparison}
\scriptsize
\setlength{\tabcolsep}{1.8pt}
\begin{tabular}{@{}lccccc@{}}
\toprule
\textbf{System} & \textbf{Partition} & \textbf{Transfer} & \textbf{Deploy} & \textbf{Cost} & \textbf{Transfer} \\
                 & \textbf{bound.}    & \textbf{size}     & \textbf{regime} & \textbf{model} & \textbf{analysis} \\
\midrule
EPD              & Stage     & GB (KV)  & NVLink   & \xmark    & \xmark \\
ModServe         & Modality  & MB (emb) & NVLink   & \xmark    & \xmark \\
SpaceServe       & Co-loc.   & ---      & Intra    & \xmark    & \xmark \\
Cauchy           & Stage     & GB (KV)  & NVLink   & Heuristic & \xmark \\
\textbf{Ours}    & \textbf{Modal.} & \textbf{MB} & \textbf{PCIe} & \textbf{Formal} & \textbf{\cmark} \\
\bottomrule
\end{tabular}
\end{table}

\section{Modality-Level Transfer Optimality}
\label{sec:insight}

The phase characterization in Section~\ref{sec:characterization} showed
that vision encoding and language decoding have different hardware optimality
regimes. The question is how this affects the choice of partition point for
cross-device execution.

A transformer decoder accumulates one
key-value pair per layer, so the KV cache grows as
$O(L \cdot s_{\text{ctx}})$. The vision encoder produces a single embedding
of size $O(N_v \cdot d)$, independent of model depth $L$. Any system that
partitions at stage boundaries must pay the $O(L)$ communication cost; any
system that partitions at the modality boundary avoids it entirely.

Section~\ref{sec:disagg} quantifies the transfer gap with concrete
numbers, Section~\ref{sec:generalization} formalizes the scaling behavior
(Theorem~\ref{thm:transfer}), and Section~\ref{sec:cost} derives a cost
model for heterogeneous deployment.

\subsection{Why Stage-Level Disaggregation Is Suboptimal}
\label{sec:disagg}

The central question for any disaggregated serving system is \emph{where} to
partition the inference pipeline. We compare three candidate partition points
with different transfer costs:

\begin{table}[h]
\centering
\caption{Transfer size per request under different disaggregation granularities
  for a 7B model with 576 vision tokens and 128 output tokens.}
\label{tab:transfer}
\small
\begin{tabular}{lcc}
\toprule
\textbf{Partition Point} & \textbf{Data Transferred} & \textbf{Size} \\
\midrule
Stage-level (PD) & KV cache & ${\sim}350$\,MB \\
Stage-level (EPD) & KV cache + tokens & ${\sim}350$\,MB \\
\textbf{Modality-level (Ours)} & \textbf{Visual embeddings} & $\mathbf{\sim 4.5}$\,\textbf{MB} \\
\bottomrule
\end{tabular}
\end{table}

\textbf{Stage-level disaggregation} (e.g., Prefill-Decode or EPD) partitions
at pipeline stage boundaries, the default choice in all prior work. After
prefill completes, the accumulated KV cache must be migrated to the decode
instance. In general, the KV cache for a single
request with context length $s_{\text{ctx}}$ requires:
\begin{equation}
D_{\text{KV}} = 2 \cdot L \cdot n_{\text{kv}} \cdot d_h \cdot s_{\text{ctx}} \cdot b
\label{eq:kv}
\end{equation}
where $n_{\text{kv}}$ is the number of KV heads, $d_h$ is the head dimension,
and $b=2$ bytes for FP16. For a 7B model with MHA ($n_{\text{kv}}{=}n_h{=}32$,
$d_h{=}128$, $L{=}32$) and $s_{\text{ctx}}{=}704$ (576 vision + 128 text):
$D_{\text{KV}} = 2 \times 32 \times 32 \times 128 \times 704 \times 2 \approx 350$\,MB.
For a typical prefill batch of 8 requests, this grows to ${\sim}2.8$\,GB.
This GB-scale transfer demands NVLink or InfiniBand bandwidth, effectively
excluding consumer GPUs from participation.

\textbf{Modality-level disaggregation} partitions at the modality
boundary, between the vision encoder output and the language model input.
The transferred data is the projected visual embedding:
\begin{equation*}
\text{Emb.\ size} = N_v \!\times\! d \!\times\! 2\,\text{B} = 576 \!\times\! 4096 \!\times\! 2 = 4.5\,\text{MB}
\end{equation*}
This represents a ${\sim}78\times$ reduction per request compared to KV cache
transfer, growing to ${\sim}600\times$ for a batch of 8.
At PCIe Gen4 x16 bandwidth (${\sim}25$\,GB/s), a 4.5\,MB transfer completes
in ${\sim}0.18$\,ms, negligible compared to the vision encoding time of
${\sim}6.8$\,s (batch=128). Even for a batch of 128 images, the total transfer
is ${\sim}576$\,MB, completing in ${\sim}23$\,ms.

This ${\sim}78\times$ gap reduces the interconnect requirement from NVLink
to PCIe, opening a design space for cross-tier heterogeneous serving that
stage-level methods cannot access.

\subsection{Formal Analysis: Transfer Ratio Scaling Law}
\label{sec:generalization}

The preceding analysis uses concrete parameters from our 7B experimental setup.
The transfer advantage of modality-level disaggregation
holds across all current MLLM architectures, and improves as models
scale.

\begin{theorem}[Transfer Optimality of the Modality Boundary]
\label{thm:transfer}
For any transformer-based MLLM with $L$ layers, $n_{\text{kv}}$ KV heads of
dimension $d_h$, hidden dimension $d = n_h \cdot d_h$, visual token count
$N_v$, and text context $s_{\text{text}}$, \textbf{under standard KV caching
semantics} (i.e., no activation recomputation, KV offloading, or speculative
decoding that alters intermediate state sizes), the ratio of stage-level
transfer (KV cache) to modality-level transfer (visual embedding) per
request is:
\begin{equation}
R = \frac{D_{\text{KV}}}{D_{\text{emb}}} = \frac{2L \cdot n_{\text{kv}} \cdot d_h}{d} \cdot \frac{s_{\text{ctx}}}{N_v}
\label{eq:ratio}
\end{equation}
where $s_{\text{ctx}} = N_v + s_{\text{text}}$.
Under MHA ($n_{\text{kv}} = n_h$), this simplifies to:
\begin{equation}
R_{\text{MHA}} = 2L\left(1 + \frac{s_{\text{text}}}{N_v}\right)
\label{eq:ratio_mha}
\end{equation}
Under GQA ($n_{\text{kv}} < n_h$), the ratio is reduced by a factor of
$n_{\text{kv}}/n_h$ but remains $\gg 1$ for all practical architectures.
Among partition points that preserve standard KV caching semantics, the
modality boundary therefore minimizes cross-device transfer complexity
asymptotically in $L$.
\end{theorem}

\begin{proof}
From Equation~\ref{eq:kv}, stage-level disaggregation transfers:
$D_{\text{KV}} = 2L \cdot n_{\text{kv}} \cdot d_h \cdot s_{\text{ctx}} \cdot b$.
Modality-level disaggregation transfers the visual embedding:
$D_{\text{emb}} = N_v \cdot d \cdot b$.
The byte width $b$ cancels, giving
$R = (2L \cdot n_{\text{kv}} \cdot d_h \cdot s_{\text{ctx}}) / (N_v \cdot d)$.
Under MHA, $n_{\text{kv}} = n_h$ and $d = n_h \cdot d_h$, so
$n_{\text{kv}} \cdot d_h / d = 1$, yielding $R = 2L \cdot s_{\text{ctx}} / N_v$.
\end{proof}

\textbf{Remark.} Theorem~\ref{thm:transfer} assumes standard KV caching.
Full activation recomputation would eliminate KV transfer but at substantial
compute cost; standard KV caching is the dominant deployment configuration.

For the MHA case (GQA preserves the same trends at reduced magnitude),
the ratio is \emph{independent of hidden dimension $d$}: both KV
cache and embedding scale linearly with $d$, so the dimension cancels.
A 7B model ($d{=}4096$) and a 70B model ($d{=}8192$) yield the same MHA
transfer ratio. The ratio also grows linearly with model depth $L$, since
deeper models accumulate proportionally more KV state while the embedding
remains a single projection, so the advantage widens with scale. Even in the worst case
($s_{\text{text}} = 0$, aggressive GQA with $n_{\text{kv}}/n_h = 0.125$),
the ratio is lower-bounded by $2L \cdot n_{\text{kv}}/n_h \geq 3$ for
$L{\geq}12$, so $R \gg 1$ for all current MLLMs.

\begin{table}[h]
\centering
\caption{Transfer ratio $R$ across representative MLLM architectures
  ($s_{\text{text}}{=}128$). For fixed-resolution models, $N_v{=}576$; for
  dynamic-resolution models (Qwen2.5-VL), $N_v{=}1024$.}
\label{tab:generalization}
\scriptsize
\setlength{\tabcolsep}{3pt}
\begin{tabular}{@{}lccc rr cc@{}}
\toprule
\textbf{Model} & $N_v$ & $L$ & $\frac{n_{\text{kv}}}{n_h}$ & $D_{\text{KV}}$ & $D_{\text{emb}}$ & $R_{\text{MHA}}$ & $R_{\text{GQA}}$ \\
\midrule
LLaVA-7B       & 576  & 32  & 32/32 & 352\,M & 4.5\,M & 78$\times$  & 78$\times$ \\
LLaVA-13B      & 576  & 40  & 40/40 & 550\,M & 5.6\,M & 98$\times$  & 98$\times$ \\
LLaVA-34B      & 576  & 60  & 8/56  & 150\,M & 7.3\,M & 147$\times$ & 21$\times$ \\
Qwen2.5-VL-7B  & 1024 & 28  & 4/28  & 83\,M  & 7.0\,M & 64$\times$  & 12$\times$ \\
Qwen-VL-72B    & 576  & 80  & 8/64  & 215\,M & 9.0\,M & 196$\times$ & 24$\times$ \\
\bottomrule
\end{tabular}
\end{table}

Even under aggressive GQA ($n_{\text{kv}}/n_h = 0.125$), the transfer ratio
remains $R \geq 21\times$, still an order of magnitude advantage for
modality-level disaggregation. The embedding size is unaffected by GQA since it
depends on $d$, not on $n_{\text{kv}}$.

\textbf{Feasibility condition for PCIe.} Modality-level disaggregation is
practical over PCIe whenever the embedding transfer time is small relative to
the vision encoding time:
\begin{equation}
\frac{T_{\text{xfer}}}{T_{\text{vision}}} = \frac{N_v \cdot d \cdot 2}{BW_{\text{PCIe}} \cdot T_{\text{vision}}} \ll 1
\end{equation}
For all models in Table~\ref{tab:generalization}, this ratio is
$< 0.003$ (less than 0.3\% overhead), confirming that PCIe is sufficient.

\textbf{Handling dynamic visual token counts.} Some modern MLLMs (e.g.,
Qwen2.5-VL) employ dynamic-resolution vision encoders that produce a variable
number of visual tokens per image depending on the input resolution. This does
not change the analysis: even
when $N_v$ varies from a few hundred to over 2000 tokens, the embedding
transfer remains MB-scale (the result holds as long as $N_v \ll L \cdot s_{\text{ctx}}$,
which is satisfied by all current vision tokenization schemes).
For instance, with $N_v{=}2048$ and
$d{=}3584$ (Qwen2.5-VL-7B), the embedding is ${\sim}14$\,MB, still
${\sim}6\times$ smaller than the corresponding KV cache for the same request.
A system exploiting modality-level disaggregation can accommodate
variable-length embeddings through dynamic buffer allocation
(\S\ref{sec:transfer}).

\textbf{Generality beyond the RTX 4090\,/\,A100 pairing.}
The $O(L)$ transfer
reduction is a \emph{model-side} property, not a hardware-side one: it
follows from the transformer architecture's accumulation of per-layer KV
state versus a single-layer embedding projection.

\textbf{(a)~Future GPU generations.} The insight applies whenever a
compute-dense accelerator (any consumer or edge GPU) is paired with a
bandwidth-rich accelerator (any datacenter GPU) over any interconnect whose
bandwidth exceeds the MB-scale embedding rate. As future consumer GPUs
increase compute density faster than interconnect bandwidth grows, the
advantage of modality-level disaggregation \emph{increases}.

\textbf{(b)~Beyond vision--language models.} The same argument
applies to any MLLM whose encoder produces a compact intermediate
representation: audio encoders (e.g., Whisper), video encoders, and
multimodal models with multiple encoder branches. Whenever the encoder
output is $O(1)$ per layer while the decoder state is $O(L)$, modality-level
disaggregation yields the same asymptotic advantage.

\textbf{(c)~Scaling trend.} As models deepen ($L$ increases for frontier
MLLMs), the $O(L)$ transfer ratio grows proportionally, making stage-level
disaggregation increasingly expensive and modality-level disaggregation
increasingly dominant. The gain therefore compounds as models deepen.

\subsection{Economic Optimality of Heterogeneous Deployment}
\label{sec:cost}

The $O(L)$ transfer reduction makes cross-tier serving \emph{feasible}. The next question is when it is also \emph{cost-efficient}. We
derive a closed-form cost model that determines when heterogeneous
deployment provably dominates homogeneous deployment as a function of
workload ratio $\rho$, cross-tier price ratio $\gamma$, and interconnect
bandwidth, independently of implementation details.

\subsubsection{End-to-End Latency Decomposition}

Under modality-level disaggregation, the end-to-end latency of a single
multimodal request decomposes as:
\begin{equation}
T_{\text{total}} = T_{\text{vision}}(\mathcal{C}) + T_{\text{xfer}} + T_{\text{prefill}}(\mathcal{D}) + T_{\text{decode}}(\mathcal{D})
\label{eq:ttotal}
\end{equation}
where each component maps to a specific hardware resource:

\begin{itemize}[leftmargin=*]
\item $T_{\text{vision}}(\mathcal{C}) = C_{\text{vision}} / (B_v \cdot R_{\text{FLOPS}}^{\mathcal{C}})$: vision encoding on a consumer GPU, where $C_{\text{vision}}$ (FLOPs) is the per-image compute cost, $B_v$ is the vision batch size, and $R_{\text{FLOPS}}^{\mathcal{C}}$ is the consumer GPU's FP16 throughput.

\item $T_{\text{xfer}} = D_{\text{emb}} / BW_{\text{PCIe}}$: embedding transfer over PCIe. For our setup, $T_{\text{xfer}} = 0.18$\,ms (negligible).

\item $T_{\text{prefill}}(\mathcal{D}) = C_{\text{prefill}}(s_{\text{ctx}}) / R_{\text{FLOPS}}^{\mathcal{D}}$: prefill is compute-bound, scaling linearly with context length.

\item $T_{\text{decode}}(\mathcal{D}) = n_{\text{out}} \cdot M_{\text{decode}} / (B_d \cdot BW_{\text{HBM}}^{\mathcal{D}})$: decode is memory-bandwidth-bound. $M_{\text{decode}}$ (bytes) is the model weight footprint streamed per token.
\end{itemize}

We treat $C_{\text{vision}}$, $C_{\text{prefill}}(\cdot)$, and
$M_{\text{decode}}$ as model-dependent constants measured via offline profiling.
We group $T_{\text{lang}} = T_{\text{prefill}} + T_{\text{decode}}$. Note that $T_{\text{vision}}$ depends on \emph{compute
throughput} (exploiting $\mathcal{C}$'s strength) while $T_{\text{lang}}$
depends on \emph{memory bandwidth} (exploiting $\mathcal{D}$'s strength).

\begin{figure}[t]
\centering
\includegraphics[width=\columnwidth]{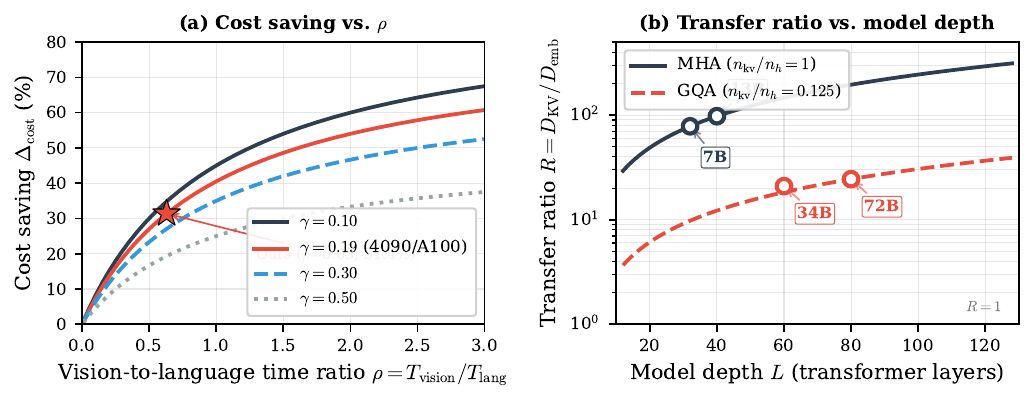}
\caption{(a)~Cost saving $\Delta_{\text{cost}}$
  (Eq.~\ref{eq:saving}) as a function of the vision-to-language time
  ratio $\rho$ for different price ratios $\gamma$. The RTX~4090/A100 operating
  point ($\gamma{=}0.19$, $\rho{=}0.63$) is marked. (b)~Transfer ratio $R$
  (Eq.~\ref{eq:ratio}) across model depths, confirming that modality-level
  disaggregation becomes increasingly advantageous for larger models.}
\label{fig:cost_model}
\end{figure}

\subsubsection{Cost Minimization}

Consider a cluster serving at throughput target $\Theta$ (requests/second):
\begin{equation}
\Theta = \min\left(\frac{n_{\mathcal{C}}}{T_{\text{vision}}},\; \frac{n_{\mathcal{D}}}{T_{\text{lang}}}\right)
\end{equation}
At the balanced operating point, $n_{\mathcal{C}} = n_{\mathcal{D}} \cdot \rho$ where $\rho = T_{\text{vision}} / T_{\text{lang}}$. The cost reduction ratio of heterogeneous over homogeneous deployment is:
\begin{equation}
\frac{\text{Cost}_{\text{hetero}}}{\text{Cost}_{\text{homo}}} = \frac{\rho \gamma + 1}{\rho + 1}
\label{eq:cost_ratio}
\end{equation}
where $\gamma = \text{Price}_{\mathcal{C}} / \text{Price}_{\mathcal{D}}$ is the cross-tier price ratio. The cost saving is:
\begin{equation}
\Delta_{\text{cost}} = \frac{\rho(1 - \gamma)}{\rho + 1}
\label{eq:saving}
\end{equation}

Savings increase with both the vision-to-language ratio $\rho$ and the price gap $(1-\gamma)$.

\textbf{Instantiation.} For our setup: $\rho \approx 0.63$, $\gamma = 3000/16000 = 0.1875$. Eq.~\ref{eq:saving} predicts $\Delta_{\text{cost}} = 31.4\%$. Our actual configuration achieves 40.6\% savings; the additional saving arises because work stealing (Section~\ref{sec:stealing}) makes consumer GPUs productive beyond their primary vision role.

\section{HeteroServe: System Design}
\label{sec:system}

We build HeteroServe to test whether the $O(L)$ transfer reduction
predicted by Theorem~\ref{thm:transfer} translates into a practical serving
system. The goal is to show that modality-level disaggregation over
commodity PCIe yields a deployable runtime with concrete engineering
solutions. HeteroServe maps compute-bound vision
encoding to consumer GPUs (RTX~4090) and bandwidth-bound language generation
to datacenter GPUs (A100), addressing three deployment challenges: low-overhead embedding transfer
(\S\ref{sec:transfer}),
load-imbalance mitigation via cross-type work stealing (\S\ref{sec:stealing}),
and runtime engine optimizations to ensure that the disaggregation
architecture, not implementation overhead, determines performance
(\S\ref{sec:engine}).

\subsection{Architectural Realization}
\label{sec:overview}

\begin{figure}[t]
\centering
\includegraphics[width=\columnwidth]{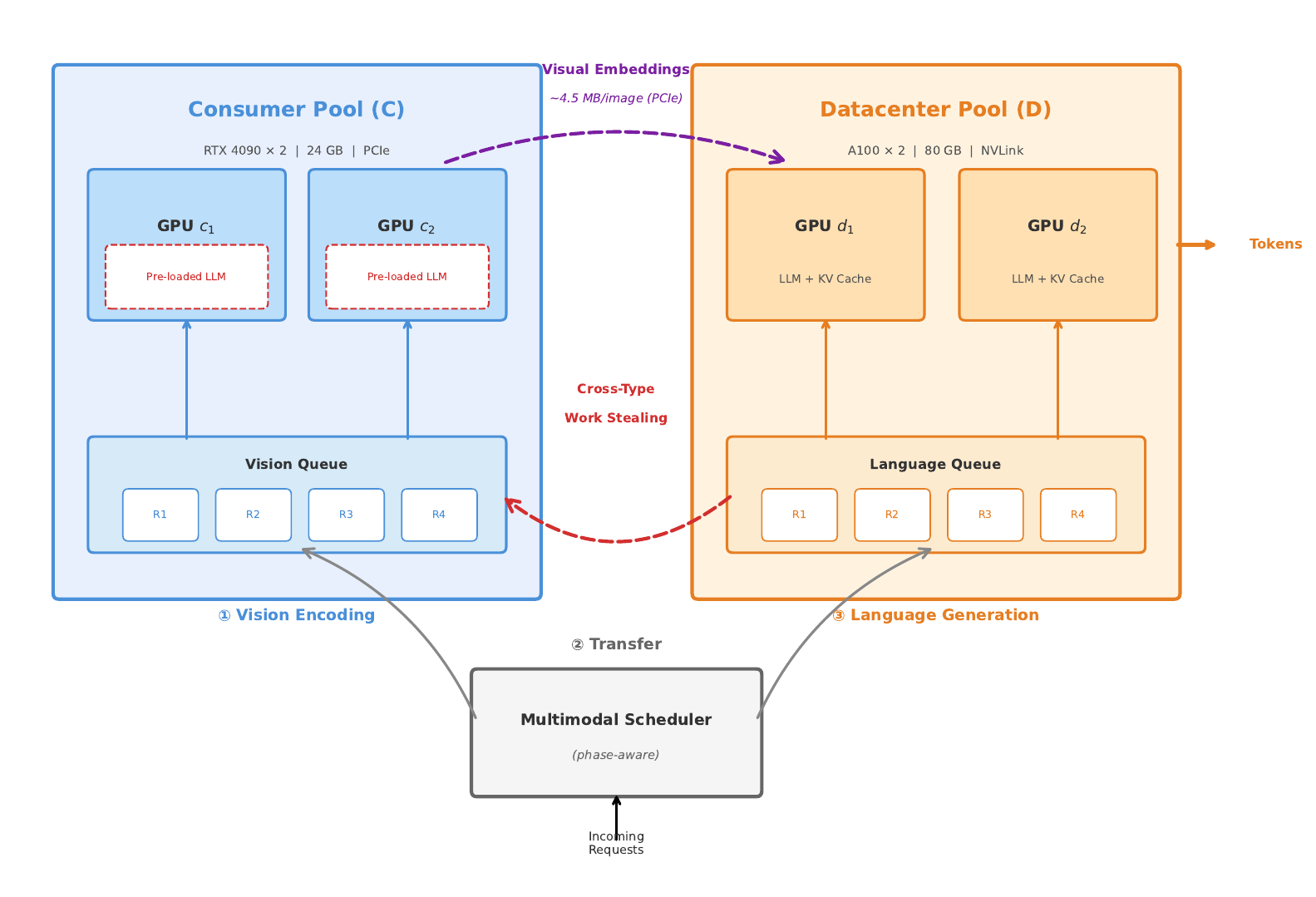}
\caption{HeteroServe architecture. Consumer GPUs (RTX~4090) handle vision
  encoding and transfer lightweight visual embeddings (${\sim}4.5$\,MB) via
  PCIe to datacenter GPUs (A100), which perform language generation. When the
  consumer pool is idle, cross-type work stealing allows consumer GPUs to
  assist with language decoding using pre-loaded LLM weights.}
\label{fig:arch}
\end{figure}

To instantiate modality-level disaggregation on real hardware, HeteroServe
organizes GPUs into two pools matched to the resource profile of each phase
(Figure~\ref{fig:arch}):

\begin{itemize}[leftmargin=*]
\item \textbf{Consumer Pool} ($\mathcal{C}$): Low-cost, high-compute GPUs
  (e.g., RTX~4090, 24\,GB VRAM, PCIe) host the vision encoder, either fixed-resolution
  (CLIP ViT-L/14, 576 tokens) or dynamic-resolution (Qwen2.5-VL, variable
  tokens), and process image encoding in parallel.
\item \textbf{Datacenter Pool} ($\mathcal{D}$): High-bandwidth GPUs (e.g.,
  A100 80\,GB, NVLink) host the language model with KV cache and run both
  prefill and decode, optionally under tensor parallelism (TP), supporting MHA
  and GQA backends.
\end{itemize}

A request flows through three phases: (1)~\textbf{Vision Encoding} on
$\mathcal{C}$, producing an embedding tensor $[1, N_v, d]$;
(2)~\textbf{Feature Transfer} via PCIe (MB-scale, ${\sim}0.18$\,ms); and
(3)~\textbf{Language Generation} on $\mathcal{D}$ (prefill + decode). A
central \emph{phase-aware Multimodal Scheduler} coordinates the pools,
maintains vision and language queues, and triggers work stealing when
$\mathcal{C}$ is idle.

\subsection{Embedding-Only Transfer Protocol}
\label{sec:transfer}

HeteroServe implements a \emph{streaming} transfer protocol that overlaps
vision encoding with feature delivery:

\begin{enumerate}
\item \textbf{Batched Vision Encoding.} The consumer GPU encodes images in
  batches. As each batch completes, the resulting embeddings are immediately
  pushed to a CPU-side pinned memory buffer via asynchronous DMA
  (\texttt{cudaMemcpyAsync} on a dedicated transfer stream).

\item \textbf{Aligned Batch Handoff.} The scheduler accumulates embeddings in
  a buffer until a target batch size $B_{\text{align}}$ (default: 32) is
  reached, then hands the entire batch to the language queue. A timeout
  mechanism ($500$\,ms) prevents tail latency for under-filled batches.
  For models with variable visual token counts (e.g., Qwen2.5-VL), the
  scheduler tracks per-request embedding lengths and pads or packs batches
  accordingly.

\item \textbf{Device Placement.} The embeddings are transferred from CPU
  pinned memory to the target datacenter GPU's HBM just before prefill begins.
\end{enumerate}

\textbf{Comparison with EPD.} EPD transfers KV cache layer-by-layer between
prefill and decode instances, requiring careful pipelining and high-bandwidth
links. Our protocol transfers a single, compact embedding tensor \emph{once}
per request, with no layer-wise coordination. This protocol is
\emph{model-agnostic}: it handles fixed-length embeddings (LLaVA) and
variable-length embeddings (Qwen2.5-VL) uniformly, since the transfer
granularity is always a single embedding per request.

\subsection{Utilization Recovery via Cross-Type Work Stealing}
\label{sec:stealing}

\begin{figure}[t]
\centering
\includegraphics[width=\columnwidth]{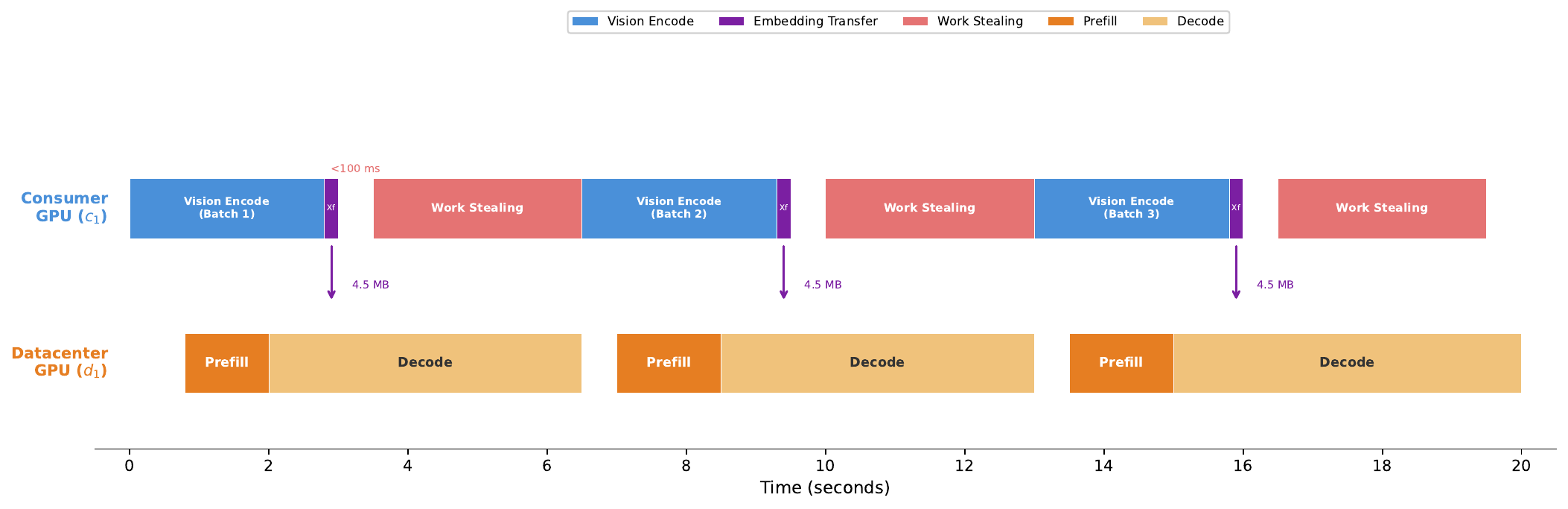}
\caption{Timeline of consumer GPU activity. After completing vision encoding,
  the consumer GPU transfers embeddings and then steals language generation
  tasks until new vision requests arrive. Pre-loaded LLM weights enable
  sub-100\,ms role switching.}
\label{fig:timeline}
\end{figure}

Heterogeneous disaggregation creates a load imbalance: the
vision encoding phase (${\sim}38\%$ of end-to-end latency) is significantly
shorter than language generation (${\sim}60\%$,
Table~\ref{tab:breakdown}), leaving consumer GPUs idle for the majority of
each serving cycle. HeteroServe recovers this stranded capacity through
\emph{cross-type work stealing}, a utilization recovery mechanism where
idle consumer GPUs temporarily assist with language decoding.
Vision requests
arrive in bursts; after encoding, consumer GPUs sit idle 60--70\% of each
cycle. Consumer GPUs steal sequence-level
language tasks (full decode sequences), minimizing scheduling overhead.
KV cache capacity (${\sim}4$\,GB after
co-residing encoder + decoder weights) limits the stealing batch size to
$B_{\text{consumer}}{=}16$.

\subsubsection{Design Principles}

\textbf{(P1) Vision-first priority.} Vision encoding is on the critical path.
Work stealing must \emph{never} delay vision processing. We enforce this by
checking the vision queue at the beginning of every scheduler iteration.

\textbf{(P2) Pre-loaded weights.} HeteroServe pre-loads the LLM decoder weights
onto consumer GPUs during initialization, trading ${\sim}6$\,GB of VRAM for
sub-100\,ms role-switching latency. The vision encoder (${\sim}400$\,MB) and
LLM decoder (${\sim}14$\,GB) co-reside on the 24\,GB consumer GPU, with
${\sim}4$\,GB remaining for KV cache during work-stealing episodes.

\textbf{(P3) Bounded assistance.} Consumer GPUs process language tasks with
smaller batch sizes ($B_{\text{consumer}}=16$ vs.\ $B_{\text{datacenter}}=64$).
Each work-stealing episode processes at most one half-batch with a $100$\,ms
collection timeout.

\subsubsection{Scheduling Algorithm}

Each consumer GPU runs a priority-driven loop (Algorithm~1).
The threshold $\tau$ (default: 16) prevents work stealing when the language
queue is not significantly backed up. This simple priority scheme is effective
because vision encoding is naturally bursty: images arrive in clusters, the
consumer GPUs process them rapidly, and then have a sustained idle period.

\begin{figure}[h]
\small
\fbox{\parbox{0.95\columnwidth}{
\textbf{Algorithm 1:} Consumer GPU Scheduling Loop\\[4pt]
\textbf{while} system is running \textbf{do}\\
\quad \textbf{if} vision\_queue is not empty \textbf{then} \hfill $\triangleright$ \textit{P1: Vision first}\\
\quad\quad batch $\gets$ collect\_vision\_batch()\\
\quad\quad encode\_and\_transfer(batch)\\
\quad \textbf{else if} language\_queue.size() $> \tau$ \textbf{then} \hfill $\triangleright$ \textit{Work stealing}\\
\quad\quad tasks $\gets$ collect\_language\_batch($B_{\text{consumer}}/2$, timeout=100ms)\\
\quad\quad generate\_tokens(tasks, consumer\_decoder)\\
\quad \textbf{else}\\
\quad\quad sleep(10ms) \hfill $\triangleright$ \textit{Idle wait}\\
\quad \textbf{end if}\\
\textbf{end while}
}}
\end{figure}

\subsubsection{Comparison with EPD's Role Switching}

EPD~\citep{epd2025} also introduces dynamic role switching within a single GPU
tier. HeteroServe's work stealing differs in three ways: (1)~it operates
\emph{across} hardware tiers; (2)~the consumer GPU performs a different
primary task (vision encoding) before stealing language work; and (3)~the asymmetric
memory constraint ($24$\,GB vs.\ $80$\,GB) requires bounded-assistance policies
that EPD does not need.

\subsubsection{Parameter Sensitivity}
\label{sec:ablation_justification}

\textbf{Aligned batch size $B_{\text{align}} = 32$.}
This parameter mediates a tradeoff between prefill throughput and queueing
latency. Too small ($\ll 32$): prefill efficiency drops below 60\% on A100.
Too large ($\gg 32$): TTFT inflates unacceptably. At 32, batch accumulation
takes ${\sim}1.7$\,s, prefill efficiency exceeds 85\%, and the CUDA Graph
captured at this size can be reused directly.

\textbf{Work stealing threshold $\tau = 16$.}
This equals $B_{\text{consumer}}$, ensuring that when stealing is triggered,
there are enough tasks to form a full consumer-side batch. Below this threshold,
the datacenter GPUs can absorb the backlog within one iteration.

\textbf{Collection timeout $= 100$\,ms.}
This bounds the maximum time a consumer GPU is locked in language mode,
providing meaningful throughput contribution (${\sim}7$ tokens) while
maintaining sub-200\,ms worst-case vision response time.

\subsubsection{Correctness Properties}
\label{sec:correctness}

The work stealing mechanism satisfies three correctness properties.

\begin{property}[No Vision Starvation]
\label{prop:starvation}
Vision tasks are never delayed by work stealing. The worst-case vision
scheduling delay from work stealing is ${\sim}212$\,ms, negligible compared to
vision encoding time (${\sim}6.8$\,s per batch of 128).
\end{property}

\begin{property}[Non-Increasing Queueing Delay]
\label{prop:bounded}
Work stealing does not increase the expected queueing delay of any language
task: $W_{\text{queue}}^{\text{steal}}(\ell) \leq W_{\text{queue}}^{\text{no-steal}}(\ell)$.
Work stealing adds an auxiliary server without removing datacenter capacity.
Under non-preemptive FIFO scheduling, adding servers cannot increase waiting
time~\citep{harchol2013performance}.
\end{property}

\begin{property}[Mode-Switching Stability]
\label{prop:stability}
The work stealing mechanism does not exhibit mode thrashing. The threshold
$\tau$ introduces hysteresis, and the half-batch limit bounds each episode to
${\sim}100$\,ms, ensuring the transition rate is bounded by the vision arrival
rate. We confirm empirically that consumer GPUs settle into a regular
alternating pattern.
\end{property}

\subsection{Runtime Engine Optimizations}
\label{sec:engine}

The preceding sections describe HeteroServe's core architectural contributions.
In this section, we describe runtime engine optimizations that, while
individually well-known, are necessary for realizing the full potential of our
architecture and explain the 54\% throughput advantage over vLLM.

\textbf{Multi-Size CUDA Graph Capture.}
We capture CUDA Graphs for multiple decode batch sizes (32 and 64) during
initialization, eliminating per-iteration CPU-GPU synchronization. CUDA Graph
replay reduces per-iteration decode latency by eliminating ${\sim}28\%$ overhead
from kernel launch and stream synchronization calls.

\textbf{Flash Attention Varlen for Packed Prefill.}
We use Flash Attention's variable-length interface, packing all sequences into
a single contiguous tensor with cumulative sequence length metadata. This
eliminates padding overhead (up to 63\%) and reduces kernel launches from $B$
to 1 per attention layer.

\textbf{Lazy KV Cache Allocation.}
HeteroServe defers KV cache block allocation until a request enters the
language prefill phase, enabling the vision queue to hold thousands of pending
requests without consuming datacenter GPU memory.

\textbf{Tensor Parallelism Support.}
For models whose memory footprint exceeds a single datacenter GPU (or to
increase decode bandwidth), HeteroServe supports tensor-parallel (TP) language
decoding across the datacenter pool. We implement a custom all-reduce kernel
based on \texttt{all\_gather} followed by local summation, which (unlike
standard NCCL \texttt{all\_reduce}) is compatible with CUDA Graph capture.
This enables CUDA Graph acceleration under TP=2 and TP=4 configurations.
The consumer pool remains unaffected by tensor parallelism: it
performs independent vision encoding regardless of how the datacenter pool
is partitioned.

\section{Experiments}
\label{sec:experiments}
\label{sec:eval}

Theorem~\ref{thm:transfer} predicts that modality-level disaggregation
reduces cross-device transfer by $O(L)$, and the cost model
(Section~\ref{sec:cost}) predicts ${\sim}31\%$ savings under practical
hardware pricing. We test both predictions on two architecturally
diverse MLLMs using HeteroServe as the evaluation platform.

\subsection{Setup}

\textbf{Hardware \& Pricing.} NVIDIA A100-80GB ($\approx$\$16k each)
and RTX 4090 ($\approx$\$3k each) at current street prices. We evaluate
multiple cluster configurations ranging from 2 to 4 GPUs.

\textbf{Models.} We select two MLLMs that differ in every major
architectural dimension to verify that HeteroServe generalizes across the
design space:
\begin{itemize}[leftmargin=*]
\item \textbf{LLaVA-1.5-7B}~\citep{liu2024llava}: uses a fixed-resolution
  CLIP ViT-L/14 vision encoder that always produces 576 visual tokens per
  image, paired with a Vicuna-7B language backbone using multi-head attention
  (MHA, 32 layers, 32 KV heads). This model can run on a single datacenter GPU.
\item \textbf{Qwen2.5-VL}: uses a dynamic-resolution vision encoder that
  produces a variable number of visual tokens depending on the input image
  (typically 256--1500+ tokens), paired with a GQA-based language backbone
  (28 layers, 4 KV heads per group). This model benefits from tensor-parallel
  (TP=2, TP=4) language decoding to achieve higher throughput.
\end{itemize}

Together, these models exercise HeteroServe's support for fixed vs.\ dynamic
visual token counts, MHA vs.\ GQA attention, and single-GPU vs.\
tensor-parallel language backends, all within the same unified framework.

\textbf{Workload.} COCO 2017 validation images with the prompt ``Describe this
image in detail'' (max output 128 tokens). We measure output token throughput
(tok/s) and define Cost-Efficiency Ratio (CER) as tok/s per \$1{,}000 hardware
investment.

\textbf{Baselines.} vLLM v0.3.0~\citep{kwon2023efficient} under homogeneous
A100 configurations with matching tensor parallelism. We configure vLLM with
identical batch sizes, tensor parallelism settings, and hardware resources to
ensure a fair comparison. vLLM represents the strongest available open-source
MLLM serving baseline at the time of evaluation.

\subsection{Cost-Efficiency with Fixed-Resolution Vision (LLaVA-1.5-7B)}

We first evaluate HeteroServe on LLaVA-1.5-7B, a model with fixed-resolution
vision encoding and single-GPU language decoding.
Table~\ref{tab:main} summarizes the results.

\begin{table}[t]
\centering
\caption{LLaVA-1.5-7B results. CER = Tokens/s per \$1{,}000 investment.}
\label{tab:main}
\small
\begin{tabular}{lcccr}
\toprule
\textbf{System Config} & \textbf{Cost} & \textbf{tok/s} & \textbf{CER} & \textbf{vs.\ vLLM} \\
\midrule
vLLM (4$\times$A100) & \$64k & 3,886 & 60.7 & Baseline \\
\midrule
Ours (Homo) & \$64k & \textbf{5,986} & 93.5 & +54\% \\
Ours (Hetero) & \$38k & 3,156 & \textbf{83.1} & +37\% CER \\
\bottomrule
\end{tabular}
\end{table}

\textbf{Cost-Efficiency.} The heterogeneous configuration (2$\times$RTX 4090 +
2$\times$A100) reduces hardware cost by \textbf{40.6\%} (\$64k $\to$ \$38k)
while delivering 81\% of baseline throughput, resulting in a \textbf{37\%
improvement} in CER.

\textbf{Throughput.} Our homogeneous implementation outperforms vLLM by
\textbf{54\%} (5,986 vs.\ 3,886 tok/s). We disentangle the sources of this
gain in Section~\ref{sec:disentangle}.

\subsection{Generalization to Dynamic-Resolution Vision (Qwen2.5-VL)}

We next evaluate HeteroServe on Qwen2.5-VL, which exercises a complementary
set of architectural features: dynamic-resolution vision encoding (variable
visual token count per image), GQA-based language attention, and tensor-parallel
decoding. The same HeteroServe framework handles these differences without
architectural modification; only the model configuration changes.
Table~\ref{tab:qwen} presents the results.

\begin{table}[t]
\centering
\caption{Qwen2.5-VL results with tensor parallelism. All language GPUs are
  A100-80GB.}
\label{tab:qwen}
\scriptsize
\setlength{\tabcolsep}{3pt}
\begin{tabular}{lcccc}
\toprule
\textbf{Configuration} & \textbf{GPUs} & \textbf{Cost} & \textbf{tok/s} & \textbf{CER} \\
\midrule
\multicolumn{5}{l}{\textit{vLLM Baselines}} \\
\quad TP=2 (2$\times$A100) & 2 & \$32k & 2,768 & 0.78 \\
\quad TP=4 (4$\times$A100) & 4 & \$64k & 3,619 & 0.51 \\
\midrule
\multicolumn{5}{l}{\textit{HeteroServe}} \\
\quad Cfg1: 4090+A100$\times$2 & 3 & \$35k & 3,252 & 0.83 \\
\quad Cfg2: 4090$\times$2+A100$\times$2 & 4 & \$38k & 3,269 & 0.78 \\
\quad Cfg3: A100$\times$2 (shared) & 2 & \$32k & 3,351 & \textbf{0.94} \\
\quad Cfg4: A100$\times$4 (shared) & 4 & \$64k & 4,633 & 0.64 \\
\bottomrule
\end{tabular}
\end{table}

\textbf{Heterogeneous configurations excel in cost-efficiency.} Cfg1
(1$\times$4090 + 2$\times$A100, \$35k) achieves CER\,=\,0.83, outperforming the
4-GPU vLLM TP=4 baseline (CER\,=\,0.51) by \textbf{63\%} at roughly half the
cost.

\textbf{Engine optimizations benefit homogeneous deployments.} Cfg3 on
2$\times$A100 achieves 3,351 tok/s, \textbf{21\% higher} than vLLM TP=2
(2,768 tok/s) on identical hardware. Cfg4 on 4$\times$A100 delivers 4,633
tok/s vs.\ vLLM's 3,619 tok/s (\textbf{+28\%}).

\textbf{Diminishing returns from scaling.} Cfg4 achieves the highest absolute
throughput (4,633 tok/s) but at \$64k, its CER (0.64) is lower than the 2-GPU
Cfg3 (0.94), demonstrating diminishing returns from GPU count scaling without
heterogeneous optimization.

\subsection{Latency Breakdown}
\label{sec:breakdown}

\begin{table}[h]
\centering
\caption{Latency breakdown for LLaVA-1.5-7B heterogeneous config (batch=128).}
\label{tab:breakdown}
\small
\begin{tabular}{lrr}
\toprule
\textbf{Phase} & \textbf{Time (s)} & \textbf{\%} \\
\midrule
Vision Encoding (RTX 4090) & 6.82 & 38.2\% \\
\textbf{Feature Transfer (PCIe)} & \textbf{0.45} & \textbf{2.5\%} \\
Language Prefill (A100) & 1.24 & 6.9\% \\
Language Decode (A100) & 9.35 & 52.4\% \\
\bottomrule
\end{tabular}
\end{table}

Table~\ref{tab:breakdown} empirically confirms the feasibility condition
from Section~\ref{sec:generalization}: the PCIe transfer overhead
(\textbf{2.5\%}) is negligible ($\ll 1$), validating that modality-level
disaggregation does not introduce a meaningful interconnect bottleneck even
over commodity PCIe.

\subsection{Work Stealing Analysis}

Enabling cross-type work stealing yields a \textbf{1.13$\times$} speedup
(3,156 vs.\ 2,793 tok/s), validating that idle consumer GPUs can meaningfully
contribute to memory-bound language generation. The consumer GPU's contribution
is bounded by its KV cache capacity (${\sim}4$\,GB), limiting the
work-stealing batch size to 16 but still providing a consistent throughput boost.

\subsection{Disentangling Architectural vs.\ Engine Gains}
\label{sec:disentangle}

HeteroServe's throughput advantage over vLLM
(54\% on LLaVA-1.5-7B, 21--28\% on Qwen2.5-VL) could arise from the
\emph{disaggregation architecture}, from \emph{engine optimizations}, or both. We
decompose the gains along these two dimensions.

\textbf{(a) Throughput ceiling (engine optimizations, identical hardware).}
Table~\ref{tab:ablation_engine} isolates the engine-level contributions on
identical 4$\times$A100 hardware. CUDA Graph capture eliminates
${\sim}28\%$ kernel launch overhead (the largest single factor), Flash
Attention Varlen removes up to 63\% padding waste in packed prefill batches,
and aligned batch scheduling increases effective language prefill batch sizes
from 4--8 to 32--48. Together these raise throughput by up to 54\% over
vLLM~v0.3.0, a prerequisite for exposing the architectural advantage,
since without them runtime overhead masks the benefit of disaggregation.

\begin{table}[h]
\centering
\caption{Engine optimization ablation on identical 4$\times$A100
  (LLaVA-1.5-7B). Each row cumulatively adds one optimization.}
\label{tab:ablation_engine}
\small
\begin{tabular}{lrr}
\toprule
\textbf{Configuration} & \textbf{tok/s} & \textbf{$\Delta$ vs.\ vLLM} \\
\midrule
vLLM v0.3.0 baseline            & 3,886 & ---        \\
\quad + CUDA Graph capture       & 4,980 & +28\%      \\
\quad + Flash Attn Varlen packed & 5,500 & +42\%      \\
\quad + Aligned batch scheduling & 5,986 & +54\%      \\
\bottomrule
\end{tabular}
\end{table}

\textbf{(b) Cost floor (architectural disaggregation, fixed budget).}
Table~\ref{tab:ablation_arch} isolates the cost-efficiency gain from
modality-level heterogeneous deployment. The heterogeneous configuration
achieves 83\% of homogeneous throughput at 59\% of the cost, yielding 37\%
higher Tokens/\$. Work stealing recovers additional idle consumer-GPU
capacity.

\begin{table}[h]
\centering
\caption{Architectural disaggregation ablation under fixed budget
  (LLaVA-1.5-7B). CER = cost-efficiency ratio (Tokens/\$).}
\label{tab:ablation_arch}
\small
\begin{tabular}{lrrr}
\toprule
\textbf{Configuration} & \textbf{tok/s} & \textbf{Cost} & \textbf{$\Delta$ CER} \\
\midrule
Homogeneous 4$\times$A100       & 5,986 & \$64k & ---        \\
Hetero.\ 2$\times$4090+2$\times$A100 & 3,156 & \$38k & +37\%      \\
\quad + Work stealing            & 3,156 & \$38k & +37\%      \\
\bottomrule
\end{tabular}
\end{table}

\textbf{Summary.} Engine improvements raise the throughput ceiling;
modality-level disaggregation reduces the cost floor. Both contribute to
cost-efficiency, but only the latter makes cross-tier heterogeneous
deployment possible.

\section{Conclusion}

We showed that, under standard transformer KV caching, the modality boundary
minimizes cross-device transfer in
multimodal inference under standard stage-based execution, reducing communication from
$O(L \cdot s_{\text{ctx}})$ to $O(N_v \cdot d)$, a factor of $O(L)$ that
is independent of hidden dimension, attention mechanism, or vision token
count, and that grows with model depth.

Because cross-device transfer drops to MB-scale, cross-tier
heterogeneous serving over commodity PCIe becomes practical. We validated
this with HeteroServe on LLaVA-1.5-7B (fixed resolution, MHA) and
Qwen2.5-VL (dynamic resolution, GQA, tensor parallel): a \$38k
heterogeneous cluster improves cost-efficiency by 37\% over a \$64k
homogeneous baseline, with up to 54\% higher throughput than vLLM on the
same hardware. As MLLMs scale to deeper architectures, the $O(L)$
transfer advantage grows proportionally, suggesting that modality-level
disaggregation will become increasingly important for cost-efficient
multimodal inference.

\bibliographystyle{plainnat}

\begin{thebibliography}{11}
\providecommand{\natexlab}[1]{#1}
\providecommand{\url}[1]{\texttt{#1}}
\expandafter\ifx\csname urlstyle\endcsname\relax
  \providecommand{\doi}[1]{doi: #1}\else
  \providecommand{\doi}{doi: \begingroup \urlstyle{rm}\Url}\fi

\bibitem[Harchol-Balter(2013)]{harchol2013performance}
Mor Harchol-Balter.
\newblock \emph{Performance Modeling and Design of Computer Systems: Queueing
  Theory in Action}.
\newblock Cambridge University Press, 2013.

\bibitem[Kwon et~al.(2023)Kwon, Li, Zhuang, Sheng, Zheng, Yu, Gonzalez, Zhang,
  and Stoica]{kwon2023efficient}
Woosuk Kwon, Zhuohan Li, Siyuan Zhuang, Ying Sheng, Lianmin Zheng, Cody~Hao Yu,
  Joseph~E Gonzalez, Hao Zhang, and Ion Stoica.
\newblock Efficient memory management for large language model serving with
  {PagedAttention}.
\newblock In \emph{Proceedings of the 29th Symposium on Operating Systems
  Principles}, pages 611--626, 2023.

\bibitem[Kwon et~al.(2024)Kwon, Li, Zhuang, Sheng, Zheng, Yu, Gonzalez, Zhang,
  and Stoica]{vllm2024}
Woosuk Kwon, Zhuohan Li, Siyuan Zhuang, Ying Sheng, Lianmin Zheng, Cody~Hao Yu,
  Joseph~E Gonzalez, Hao Zhang, and Ion Stoica.
\newblock {vLLM}: Easy, fast, and cheap {LLM} serving with {PagedAttention}.
\newblock \emph{arXiv preprint arXiv:2309.06180}, 2024.

\bibitem[Li et~al.(2025)Li, Zhang, Zhao, Li, Shi, Zhang, Li, Yu, Yang, Wen, and
  Cui]{spaceserve2025}
Zhicheng Li, Shuoming Zhang, Jiacheng Zhao, Siqi Li, Xiyu Shi, Yangyu Zhang,
  Shuaijiang Li, Donglin Yu, Zheming Yang, Yuan Wen, and Huimin Cui.
\newblock {SpaceServe}: Efficient kernel-level multiplexing for large language
  model inference.
\newblock In \emph{Proceedings of the 39th Conference on Neural Information
  Processing Systems (NeurIPS)}, 2025.

\bibitem[Liu et~al.(2024)Liu, Li, Li, and Lee]{liu2024llava}
Haotian Liu, Chunyuan Li, Yuheng Li, and Young~Jae Lee.
\newblock Improved baselines with visual instruction tuning.
\newblock \emph{arXiv preprint arXiv:2310.03744}, 2024.

\bibitem[Ma et~al.(2025)]{cornserve2025}
Jeff~J. Ma et~al.
\newblock {CornServe}: Efficiently serving any-to-any multimodal models.
\newblock \emph{arXiv preprint arXiv:2512.14098}, 2025.

\bibitem[Mo et~al.(2025)Mo, Liao, Xu, Zhou, and Xu]{hetis2025}
Zizhao Mo, Jianxiong Liao, Huanle Xu, Zhi Zhou, and ChengZhong Xu.
\newblock {Hetis}: Serving {LLMs} in heterogeneous {GPU} clusters with
  fine-grained and dynamic parallelism.
\newblock In \emph{Proceedings of the International Conference for High
  Performance Computing, Networking, Storage and Analysis (SC)}, pages
  1710--1724, 2025.
\newblock \doi{10.1145/3712285.3759784}.

\bibitem[Qiu et~al.(2025)]{modserve2025}
Haoran Qiu et~al.
\newblock {ModServe}: Modality- and stage-aware resource disaggregation for
  scalable multimodal model serving.
\newblock \emph{arXiv preprint arXiv:2502.00937}, 2025.

\bibitem[Singh et~al.(2025)Singh, Wang, Hu, Yu, Xing, Jiang, Wang, Bai, Li,
  Xiong, Zhang, and Fan]{epd2025}
Gursimran Singh, Xinglu Wang, Yifan Hu, Timothy Yu, Linzi Xing, Wei Jiang,
  Zhefeng Wang, Xiaolong Bai, Yi~Li, Ying Xiong, Yong Zhang, and Zhenan Fan.
\newblock Efficiently serving large multimodal models using
  encode-prefill-decode ({EPD}) disaggregation.
\newblock \emph{arXiv preprint arXiv:2501.05460}, 2025.

\bibitem[Zhang et~al.(2025)Zhang, Shen, Yang, Tian, Luo, Zhang, Li, Hu, Wo,
  Song, and Ouyang]{cauchy2025}
Yihui Zhang, Han Shen, Renyu Yang, Di~Tian, Yuxi Luo, Menghao Zhang, Li~Li,
  Chunming Hu, Tianyu Wo, Chengru Song, and Jin Ouyang.
\newblock {Cauchy}: A cost-efficient {LLM} serving system through adaptive
  heterogeneous deployment.
\newblock In \emph{Proceedings of the 2025 {ACM} Symposium on Cloud Computing
  (SoCC)}, pages 881--893, 2025.
\newblock \doi{10.1145/3772052.3772264}.

\bibitem[Zhao et~al.(2025)]{unifiedserve2025}
Lingxiao Zhao et~al.
\newblock Enabling disaggregated multi-stage {MLLM} inference via
  {GPU}-internal scheduling and resource sharing.
\newblock \emph{arXiv preprint arXiv:2512.17574}, 2025.

\end{thebibliography}

\end{document}